%% file: main.tex
\documentclass[smallcaptions]{dccpaper}
\usepackage{amssymb,amsthm,booktabs,graphicx,microtype}
\usepackage[hidelinks]{hyperref}
\hypersetup{pdftitle={RangeRectQuant: Compressing Scientific Fields for Every Future Rectangular Sum},pdfauthor={Jiahao Zhang, Xinling Wen, Yu Chen}}
\usepackage{cite,needspace,placeins}

\newtheorem{theorem}{Theorem}
\newtheorem{lemma}{Lemma}
\newtheorem{proposition}{Proposition}
\newcommand{\Z}{\mathbb Z}
\newcommand{\RRQ}{RangeRectQuant}
\begin{document}
\title{\textbf{RangeRectQuant: Compressing Scientific Fields\newline for Every Future Rectangular Sum}}
\author{Jiahao Zhang, Xinling Wen, Yu Chen$^{*}$\\
Zhengzhou University of Aeronautics\\
$^{*}$Corresponding author: \texttt{chenyu@zua.edu.cn}}
\maketitle
\begin{abstract}
\input{sections/abstract}
\end{abstract}
\input{sections/introduction}
\input{sections/related_work}
\input{sections/method}
\input{sections/evaluation}
\input{sections/conclusion}
\Section{References}
\bibliographystyle{IEEEbib}
\bibliography{references}
\end{document}

%% file: sections/abstract.tex
Scientific rasters are often stored before the regions to be measured are known. Independent pixel-error bounds can force exact integer reconstruction when one absolute error budget covers every future rectangle. Coordinated changes allow every pixel to change while every rectangular sum remains accurate. A prefix sum accumulates values from the top-left corner. Four such sums determine any rectangular total. RangeRectQuant preserves the pixel range and designated exact samples while controlling errors in prefix sums. At each row position, we characterize exactly which errors let the remaining pixels meet these constraints. These errors form a single integer interval, even when many completions exist. A backward pass finds its endpoints; a forward pass uses them to choose quantized pixels. The resulting field meets the prescribed error budget for every axis-aligned rectangle, with linear work and workspace proportional to one row. Experiments on geographically held-out built-surface and water-occurrence maps show substantial storage savings over alternatives meeting the same requirements. Data producers encode once; later users measure arbitrary rectangles without the source, re-encoding, or query-specific correction.

%% file: sections/introduction.tex
\Section{Introduction}
Scientific rasters are archived before future measurement regions are known. Later users total values over neighborhoods, tiles, or other rectangles. Built-surface maps record building area per pixel. We seek one compressed field that keeps every rectangular sum within a prescribed error budget.

A natural approach gives each pixel the same independent error bound. Errors from many pixels can accumulate in the same direction. Certifying a small regional budget this way can force exact integer reconstruction (Proposition~\ref{prop:independent}). Coordinated errors allow every pixel to change: turning ones into alternating zeros and twos changes any rectangular sum by at most one (Figure~\ref{fig:coordination}).

A prefix sum adds values from the top-left corner to a pixel; four such sums determine any rectangular total~\cite{asano2003rounding,lyu2023sigmadelta}. Bounded prefix errors give an area-independent guarantee. Pixel ranges complicate coordination: zero cannot decrease, the maximum cannot increase, and a current choice can prevent the row from finishing legally.

\begin{figure}[htbp]
 \centering
 \includegraphics[width=\linewidth]{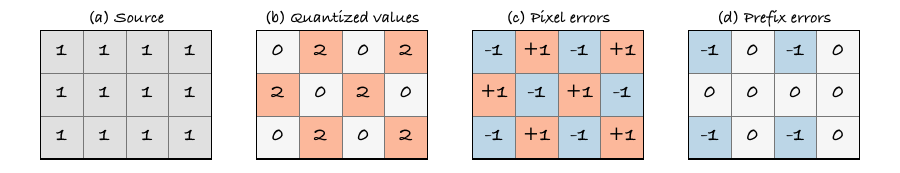}
 \caption{Coordination on a field of ones, with valid range zero to two. Every pixel changes; every rectangle error has magnitude at most one. Prefix errors stay between minus one and zero.}
 \label{fig:coordination}
\end{figure}

\RRQ\ looks ahead before committing a pixel (Figure~\ref{fig:overview}). All prefix-error states that can reach a legal row completion form a single integer interval. A backward pass finds its endpoints; a forward pass uses it to choose compressible pixels. Linear work guarantees every rectangle, including over 275 billion in a $1024\times1024$ field.

%% file: sections/related_work.tex
\Section{Related Work}
Pointwise distortion is a central target of scientific compression~\cite{di2025survey}. Prediction and interpolation methods improve storage rates while controlling individual reconstruction errors~\cite{sz32022,liu2022qoz,liu2024hpez,shi2024machete}. Quantization-error coding~\cite{lu2023adtfse} and integer representations for astronomical images~\cite{maireles2022astronomy} also improve rate--precision tradeoffs. A uniform pixel tolerance gives each sample the same freedom to change. Requiring every regional sum to share one small budget can remove that freedom entirely (Proposition~\ref{prop:independent}).

Specified downstream quantities bring compression closer to the analysis users perform~\cite{cappello2025practice}. These include derived measurements bounded through field norms~\cite{mgard2020,mgard2019qoi}, calculations preserved through tailored pixel tolerances~\cite{jiao2022qoi,liu2025qpet}, and selected ranges, regions, or topological features~\cite{liu2022constraints,yan2024toposz}. Query-aware representations optimize mean-squared range-query error~\cite{matias2006synopses} or predict bias and uncertainty~\cite{topiq2026}. Archived rasters pose a different requirement: the measurement regions are unknown at encoding time, and every later rectangle must satisfy a deterministic error budget.

Coordinated rounding supplies the closest structural foundation. Matrix rounding controls regional errors for two nested or disjoint region families~\cite{asano2003rounding}. Two-dimensional sigma--delta quantization controls errors in prefix sums~\cite{lyu2023sigmadelta}. Its stable greedy scheme uses a finite expanded alphabet whose values may leave the input range. Table~\ref{tab:capabilities} makes the combined requirement visible. \RRQ\ constructs one bounded integer field that supports every future rectangle, preserves designated exact samples, and needs no source at query time. It achieves this combination by characterizing exactly which cumulative errors remain realizable by valid pixels.

\begin{table}[htbp]
 \centering
 \caption{Guarantees of representative constructions. A checkmark denotes a supplied guarantee; a dash denotes no supplied guarantee. $^*$Uniform pixel bounds require exact reconstruction in the worst case with at least two allowed integers and more changeable pixels than the budget. Clipping and integer rounding do not inherit the original sum budget.}
 \label{tab:capabilities}
 \begin{tabular}{p{.32\linewidth}ccc}
 \toprule
 Construction & \shortstack{Regional\\sums} & \shortstack{Valid integer\\pixels} & \shortstack{Designated\\exact values}\\
 \midrule
 Uniform pixel bound & Exact only$^*$ & $\checkmark$ & $\checkmark$\\
 Stable sigma--delta & $\checkmark$ & --- & ---\\
 Clip and integer-round & --- & $\checkmark$ & ---\\
 \RRQ & $\checkmark$ & $\checkmark$ & $\checkmark$\\
 \bottomrule
 \end{tabular}
\end{table}

%% file: sections/method.tex
\Section{Constructing a Field with Accurate Regional Sums}
\SubSection{Problem definition}
Let $X$ be a field of $H$ rows and $W$ columns, with integer samples in $[0,U]$, where $U$ is a nonnegative integer. Indices start at one. A mask marks missing samples, represented by zero in both $X$ and the reconstruction $Y$. The protected set $P$ includes these samples and any designated exact values. For a finite nonnegative budget $\varepsilon$, we require
\begin{equation}
 Y_{ij}\in\{0,\ldots,U\},\qquad Y_{ij}=X_{ij}\ ((i,j)\in P),\qquad
 \max_R\left|\sum_{(i,j)\in R}(Y_{ij}-X_{ij})\right|\leq\varepsilon,
 \label{eq:contract}
\end{equation}
where $R$ ranges over all nonempty axis-aligned rectangles. Exact transmission of the mask also preserves each region's valid-sample count. A mean over $k>0$ valid samples then has error at most $\varepsilon/k$.

Let $A$ count unprotected samples and $\delta\geq0$ be a finite uniform pixel tolerance.
\begin{proposition}[Independent pixel bounds]
For integer source and output values in $[0,U]$, the sharp worst-case rectangular error under independent pixel tolerance $\delta$ is $A\min(\lfloor\delta\rfloor,U)$. If $U\geq1$ and $\varepsilon<A$, these bounds certify the budget only with exact reconstruction.
\label{prop:independent}
\end{proposition}
\begin{proof}
Set $d=\min(\lfloor\delta\rfloor,U)$. Each change has magnitude at most $d$. A zero source and output $d$ at every unprotected sample attain $Ad$ over the whole field. For $\varepsilon<A$, integer $d$ must be zero.
\end{proof}

Define the prefix error $E_{ij}$ as the sum of $Y-X$ over rows $1$ through $i$ and columns $1$ through $j$. Set $E_{0j}=E_{i0}=0$, and let $B=\lfloor\varepsilon/4\rfloor$. For $R=[a,b]\times[c,d]$,
\begin{equation}
 \sum_R(Y-X)=E_{bd}-E_{a-1,d}-E_{b,c-1}+E_{a-1,c-1}.
 \label{eq:corners}
\end{equation}
\Needspace{3\baselineskip}
Thus $|E_{ij}|\leq B$ implies the regional budget in~\eqref{eq:contract}, independently of region size. Our construction keeps this bound while selecting values of $Y$ directly.

Figure~\ref{fig:coordination} shows how coordination changes the possibilities. For $X_{ij}=1$, $U=2$, and $P=\varnothing$, set $Y_{ij}=1-(-1)^{i+j}$. Every pixel changes. Each rectangle error factors into two alternating sums, each of magnitude at most one. For $HW>4$, independent tolerance at least one permits an all-twos field with error $HW$, so certifying budget four requires exact pixels.

\begin{figure}[htbp]
 \centering
 \includegraphics[width=\linewidth]{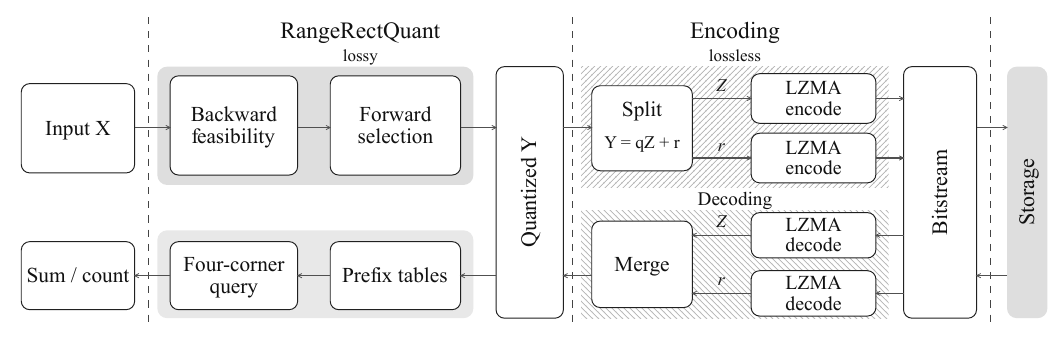}
 \caption{Overview of \RRQ: input $X$ becomes quantized field $Y$. For integer step $q>0$, quotient $Z=\lfloor Y/q\rfloor$ and remainder $r=Y-qZ$ use lossless Lempel--Ziv--Markov chain algorithm (LZMA) coding. Decoding restores $Y$; prefix tables serve sums and valid counts.}
 \label{fig:overview}
\end{figure}

\SubSection{A row transition that includes the value range}
The row becomes a path through prefix errors. The preceding row supplies incoming errors; each new state determines a pixel and must respect its range.

In row $i$, the source is $x_j=X_{ij}$ and the incoming error is $u_j=E_{i-1,j}$, with $u_0=0$. The previous state is $a=E_{i,j-1}$; the new state is $b=E_{ij}$. Incoming errors change by $\delta_j=u_j-u_{j-1}$. Setting $t_j=x_j-\delta_j$ gives reconstructed value $y_j=t_j+b-a$.
The prefix bound requires $b\in[-B,B]$, the range requires $b\in[a-t_j,a+U-t_j]$, and a protected pixel forces $b=a+\delta_j$.

We favor values divisible by a positive integer step $q$, called \emph{coarse pixels}. A \emph{fine pixel} retains a nonzero remainder modulo $q$. This transition expresses the coarse condition as
\begin{equation}
 b\equiv a-t_j\pmod q.
 \label{eq:coarse}
\end{equation}
Copying the source row gives a feasible path: $a=u_{j-1}$ and $b=u_j$ produce $y_j=x_j$.

\SubSection{Exact intervals for completing a row}
The remaining row may have many legal completions. All prefix-error states that can reach one form a single integer interval.

Only its two endpoints are needed to describe every feasible starting state. Computing them backward lets each forward choice retain a way to finish the row.

Let $V_j$ contain the integer states after column $j$ that allow the remaining row to be completed. Write $\Z$ for the integers and $P_i=\{j:(i,j)\in P\}$. Starting with $V_W=[-B,B]\cap\Z$, each interval $V_j=[L_j,H_j]\cap\Z$ has predecessor
\begin{equation}
 V_{j-1}=
 \begin{cases}
 [-B,B]\cap[L_j+t_j-U,H_j+t_j]\cap\Z, & j\notin P_i,\\
 [-B,B]\cap[L_j-\delta_j,H_j-\delta_j]\cap\Z, & j\in P_i.
 \end{cases}
 \label{eq:backward}
\end{equation}
A backward pass computes these endpoints. During the forward pass, the unprotected pixel's allowed states are
\begin{equation}
 I_j=V_j\cap[a-t_j,a+U-t_j].
 \label{eq:forward}
\end{equation}
If $I_j$ contains a state satisfying~\eqref{eq:coarse}, choose the one closest to zero, breaking ties toward the smaller state. Otherwise choose the integer in $I_j$ closest to zero. A protected pixel takes its forced state. Update $a$ after each pixel and $u$ after each row.

\begin{lemma}[Exact remaining-row realizability]
For preceding-row states $u_j\in[-B,B]$, the sets in~\eqref{eq:backward} are exactly those admitting a legal completion. Each is an integer interval, and $0\in V_0$.
\end{lemma}
\begin{proof}
An unprotected successor $b$ permits predecessors in $[b+t_j-U,b+t_j]\cap\Z$. Their union over consecutive $b\in V_j$ gives~\eqref{eq:backward}, including $U=0$. Protected transitions translate intervals. Backward induction gives exactness. The legal source suffix supplies $u_j\in V_j$, including $0=u_0\in V_0$.
\end{proof}

\begin{theorem}[Linear-time realization of accurate rectangular sums]
For every bounded integer input, protected set, integer $q\geq1$, and finite $\varepsilon\geq0$, the backward--forward construction terminates and satisfies~\eqref{eq:contract}. It takes $O(HW)$ arithmetic operations and $O(W)$ additional row workspace.
\label{thm:construction}
\end{theorem}
\begin{proof}
Exact feasibility supplies a successor from each $a\in V_{j-1}$. Every transition enforces range and protection. Induction gives $|E_{ij}|\leq B$ and~\eqref{eq:contract}. Two scans per row use constant-work interval and residue operations with row-length arrays.
\end{proof}

The $H(H+1)W(W+1)/4=\Theta(H^2W^2)$ possible rectangles share this guarantee after $O(HW)$ arithmetic operations. Workspace excludes input, output, and verification tables.

\begin{lemma}[Feasible intervals determine fine precision]
Write $I_j=[l,h]\cap\Z$ at an unprotected forward choice. Its feasible pixel interval is $J_j=[s,v]\cap\Z$, where $s=t_j+l-a$ and $v=t_j+h-a$. It contains $\kappa_j=\lfloor v/q\rfloor-\lfloor(s-1)/q\rfloor$ coarse values. Let $m$ count fine pixels in the field, and $m_P$ count protected source values with nonzero remainder. Then
\begin{equation}
 m=m_P+\sum_{(i,j)\notin P}\mathbf{1}\{\kappa_j=0\}
 \leq m_P+\sum_{(i,j)\notin P}\mathbf{1}\{|I_j|<q\}.
 \label{eq:fine-count}
\end{equation}
\end{lemma}
\begin{proof}
The formula counts multiples of $q$ in $J_j$, and the forward rule selects one whenever present. Any $q$ consecutive integers contain a multiple. Protected outputs equal their source values.
\end{proof}
Figure~\ref{fig:spatial} maps these intervals and fine pixels in Cairo, with prefix errors bounded throughout.

\SubSection{Encoding the constructed field}
Write $Y=qZ+r$, where $Z=\lfloor Y/q\rfloor$ and $0\leq r<q$. LZMA codes both 16-bit arrays after grouping low and high bytes. Files also store masks and metadata.

The intervals determine which pixels can use coarse values. Fine positions and remainders then contribute to encoding length. Set $N=HW$ and let $K=\lfloor U/q\rfloor+1$ count the possible quotient values; $m$ is the fine count in~\eqref{eq:fine-count}.
\begin{proposition}[Geometry-dependent representation length]
For $q\geq2$ and known $N,U,q$, a lossless representation of $Y$ requires at most
\begin{equation}
 \lceil N\log_2K\rceil+\lceil\log_2(N+1)\rceil
 +\left\lceil\log_2\binom{N}{m}\right\rceil
 +\lceil m\log_2(q-1)\rceil
 \label{eq:finite-code}
\end{equation}
bits, plus masks and metadata. For $q=1$, $m=0$ and the bound is $\lceil N\log_2(U+1)\rceil$.
\end{proposition}
\begin{proof}
Encode quotient, count, fine positions, and remainders by ranks among $K^N$, $N+1$, $\binom Nm$, and $(q-1)^m$ possibilities. Each rank uses the ceiling of its set-size logarithm. Decoding forms $Y=qZ+r$; at $q=1$, only the quotient remains.
\end{proof}
Coarse alphabets reduce quotient cost; fine pixels add positions and remainders.

\Needspace{6\baselineskip}
The band $B$ sets accuracy; $q$ controls storage. For finite nonempty step set $\mathcal Q$, let $L_q(X)$ be the encoded file length at $q$. Select
\begin{equation}
 q^*=\operatorname*{arg\,min}_{q\in\mathcal Q}L_q(X)
 \label{eq:step-selection}
\end{equation}
All candidates retain Theorem~\ref{thm:construction}'s guarantee. Ties favor smaller $q$. Each trial repeats construction and coding.

\begin{figure}[htbp]
 \centering
 \includegraphics[width=\linewidth]{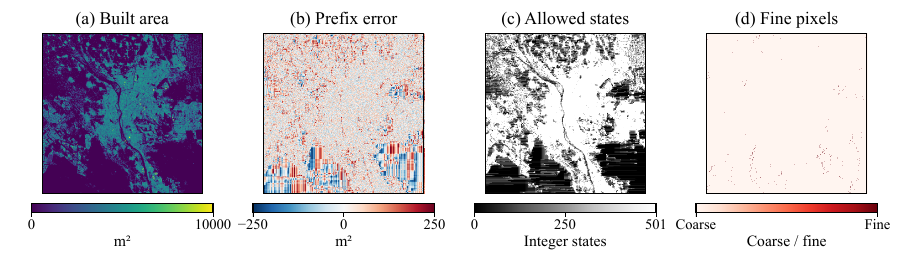}
 \caption{A $512^2$ Cairo built-area field with 100\,m pixels~\cite{ghsl2023}, at budget $\varepsilon=1000$\,m$^2$ and step $q=200$\,m$^2$. (a) Source values. (b) Prefix error bounded by $\pm250$\,m$^2$. (c) Allowed prefix-error states per pixel; darker regions offer fewer choices. (d) Fine pixels retaining nonzero remainders. All panels share the same grid.}
 \label{fig:spatial}
\end{figure}

%% file: sections/evaluation.tex
\Section{Evaluation}
\SubSection{Data and file comparisons}
We compare self-contained files of two bounded integer rasters: built surface measures building area per pixel; water occurrence records how often water is observed.

Four locations are used to choose the fixed encoding configuration. Four separate cities---Cairo, Delhi, New York, and Jakarta---illustrate the construction and step tradeoff. The main built-area rate study uses forty additional locations in six continental groups, selected from metadata before pixel access. Native $1024^2$ windows supply central $512^2$ and $256^2$ crops. Water uses an independent set of eight metadata-selected $512^2$ windows.

Built area uses the 2020 Global Human Settlement Layer: integers in $[0,10000]$\,m$^2$ on a native 100\,m grid~\cite{ghsl2023}. Water occurrence uses the 1984--2024 product~\cite{pekel2016water,gsw2026}: values in $[0,100]$, valid zero, and missing code 255. Both retain missing samples. Figure~\ref{fig:frontier} shows eight sum budgets per quantity: 25--5000\,m$^2$ for built area and 5--1000 percentage points for water. The fixed step is $q=\max(1,\lfloor\varepsilon/5\rfloor)$, capped at 100 for water.

A \emph{valid self-contained file} decodes without the source, preserves the missing mask and protected samples, and meets the integer range and rectangular error budget in~\eqref{eq:contract}. File sizes include masks, configuration, geospatial metadata, digest, and checksum. We compare six lossless integer representations, lossless JPEG\,2000~\cite{marcellin2000jpeg}, and SZ3, MGARD, and QPET~\cite{sz32022,mgard2020,liu2025qpet}. Water adds native-byte LZMA and Portable Network Graphics (PNG) files.

We sweep representative supported configurations of SZ3, MGARD, and QPET. Each decoded field is converted to the required bounded-integer representation and evaluated under the same regional-sum requirement. Files are decoded without access to the source.

The fixed-margin control changes $2\times2$ blocks by $(d,-d;-d,d)$, preserving row and column sums at nine integer bounds. A \emph{sum-pyramid control} splits square totals among four children within valid ranges. Its files store pixels or deviations of child totals from proportional parent shares.

\Needspace{12\baselineskip}
\SubSection{Regional accuracy and valid pixels}
\begin{figure}[htbp]
 \centering
 \includegraphics[width=\linewidth]{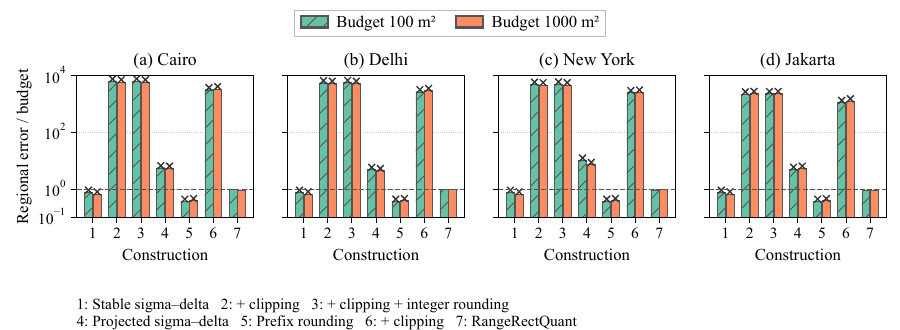}
 \caption{Regional accuracy and valid pixels. Bars give maximum rectangle error divided by budget; the dashed line marks the target. Crosses mark invalid pixels or over-budget errors. \RRQ\ meets both requirements.}
 \label{fig:closest}
\end{figure}
Regional accuracy and valid pixels must hold together (Figure~\ref{fig:closest}). Stable greedy two-dimensional sigma--delta~\cite{lyu2023sigmadelta} stores rational outputs exactly. Variants clip, integer-round, constrain its alphabet, or round prefixes with optional clipping. Stable sigma--delta controls sums with 2.22--6.76\% out-of-range pixels; clipping raises regional error to thousands of times the budget. \RRQ\ incorporates the range into its choices, delivering both properties in all eight cases.

\FloatBarrier
\SubSection{Rate across locations and budgets}
\begin{figure}[htbp]
 \centering
 \includegraphics[width=\linewidth]{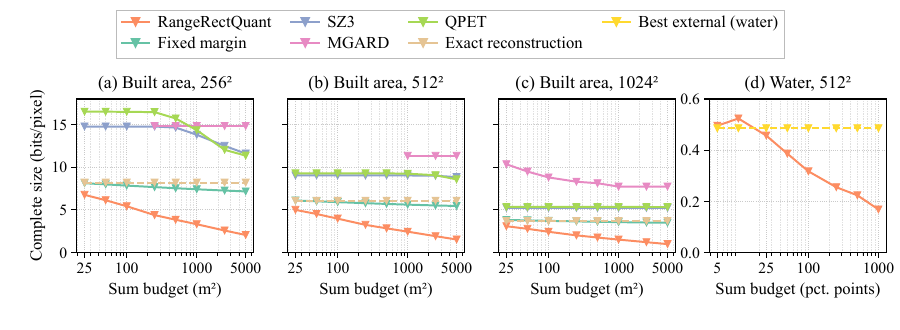}
 \caption{Self-contained rates. (a--c) Forty built-area locations at three sizes. Curves require valid files everywhere; gaps mark incomplete coverage. (d) Eight water locations; budgets use percentage-point sums.}
 \label{fig:frontier}
\end{figure}
All 960 proposed built-area files change pixels while preserving the target guarantees. Each is smaller than the smallest valid lossless, scientific-compressor, or fixed-margin file, with aggregate reductions of 16.23--72.17\% (Figure~\ref{fig:frontier}).

Adding valid sum-pyramid files retains forty wins at budgets 100 and 1000, with savings of 32.55\% and 41.30\% (Table~\ref{tab:rates}).

\begin{table}[t]
 \centering
 \caption{File rates in bits per pixel, including masked positions. Reductions compare summed bytes with the smallest valid external files, including $512^2$ sum-pyramid controls. Budgets use m$^2$ for built area and percentage-point sums for water.}
 \label{tab:rates}
 \begin{tabular}{lrrrrr}
 \toprule
 Quantity and size & $\varepsilon$ & External & \RRQ & Reduction & Wins\\
 \midrule
 Built area, $256^2$ & 100 & 7.818 & 5.421 & 30.65\% & 40/40\\
 & 1000 & 7.398 & 3.311 & 55.24\% & 40/40\\
 Built area, $512^2$ & 100 & 5.878 & 3.965 & 32.55\% & 40/40\\
 & 1000 & 4.153 & 2.438 & 41.30\% & 40/40\\
 Built area, $1024^2$ & 100 & 3.646 & 2.424 & 33.50\% & 40/40\\
 & 1000 & 3.535 & 1.506 & 57.39\% & 40/40\\
 Water, $512^2$ & 100 & 0.486 & 0.317 & 34.66\% & 7/8\\
 & 1000 & 0.486 & 0.168 & 65.35\% & 8/8\\
 \bottomrule
 \end{tabular}
\end{table}
Water savings begin at budget 25 and reach 65.35\% at 1000, with all eight locations winning from 250. Native-byte lossless files are smaller at 5 and 10.

\FloatBarrier
\SubSection{Rate decomposition}
\begin{figure}[tbp]
 \centering
 \includegraphics[width=\linewidth]{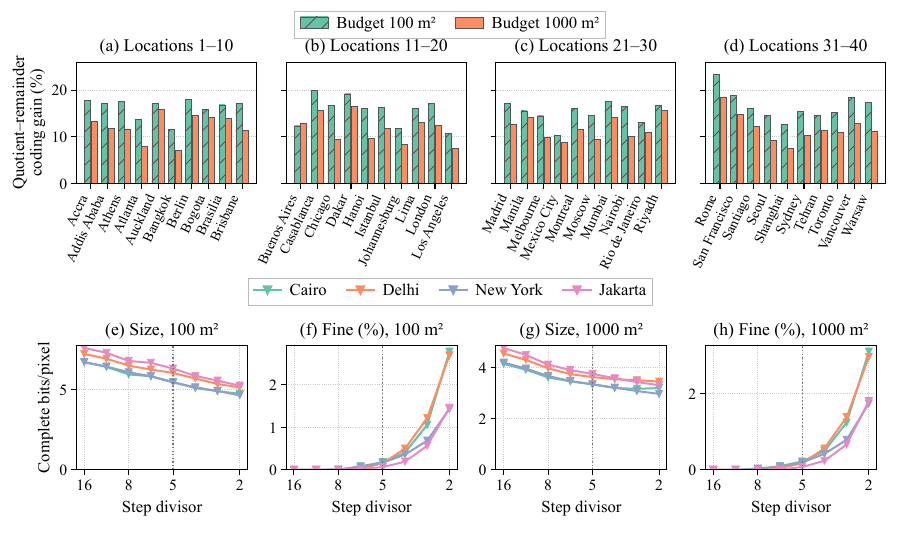}
 \caption{Storage and precision. (a--d) Quotient--remainder file-size reductions versus direct coding of identical $512^2$ fields at forty alphabetically ordered locations. (e--h) Four separate cities at eight step divisors; $q=\max(1,\lfloor\varepsilon/\text{divisor}\rfloor)$. Dotted lines mark divisor five.}
 \label{fig:representation}
\end{figure}
Quotient--remainder coding saves 15.20\% and 11.03\% over direct 16-bit LZMA of identical fields at budgets 100 and 1000, also beating byte-shuffled coding (Figure~\ref{fig:representation}). Remainders cost 0.070 and 0.086 bits per pixel; removing them exceeds the budget in all 320 fields.

Step sweeps on the four configuration locations and four separate cities preserve all constraints, with fine choices matching~\eqref{eq:fine-count}. An eight-step probe on the four cities uses common metadata and~\eqref{eq:step-selection}, saving a further 15.04\% and 8.30\% at budgets 100 and 1000. Search averages 1.20\,s per field. Main rate comparisons retain the fixed step.

\SubSection{Protected samples and scalability}
Forty configurations with nested protected sets of up to 10\% retain exact samples, pixel range, and regional accuracy. Warm construction with verification takes 22--25\,ms.

Native $4096^2$ London fields run on a 16\,GB Apple M4. Three warmed, single-thread repetitions at budget 100 give medians of 7.84\,s encoding and 1.14\,s decoding. Construction alone reaches 18.39 million pixels per second.

%% file: sections/conclusion.tex
\FloatBarrier
\Needspace{6\baselineskip}
\Section{Conclusion}
Uniform independent pixel bounds can force exact reconstruction. Exact completion intervals make coordinated quantization linear, preserving the pixel range and accuracy of every rectangular sum. Built-surface and water-occurrence studies demonstrate file savings. \RRQ\ shows how cumulative query structure can guide compression directly: represent future measurements through compact states, characterize which states remain realizable, and quantize within them.

The repository provides code, data, core experiments, and extended figures:\newline
\url{https://github.com/Marchematics/RangeRectQuant}.